\documentclass[preprint,3p]{elsarticle}
\usepackage{lipsum}
\makeatletter
\def\ps@pprintTitle{%
 \let\@oddhead\@empty
 \let\@evenhead\@empty
 \def\@oddfoot{}%
 \let\@evenfoot\@oddfoot}
\usepackage{lineno,hyperref}
\modulolinenumbers[5]
\usepackage{amsmath,amsfonts,amssymb,graphicx}
\usepackage{hyperref}
\usepackage{color}
\usepackage{xcolor}
\usepackage{mathrsfs}
\usepackage[utf8]{inputenc}
\usepackage[english]{babel} 
\usepackage{tikz}
\usetikzlibrary{arrows}
\usetikzlibrary{positioning,fit,calc}
\tikzset{block/.style={draw,thick,text width=1.5cm,minimum height=1cm,align=center},
line/.style={-latex} }
\usepackage{lipsum,adjustbox}
\usepackage{bookmark}
\usepackage{multirow}

\usepackage{array}
\newcolumntype{M}
{>{\centering\arraybackslash}m{1cm}}
\usepackage{hyperref}
\usepackage{color}
\usepackage{xcolor}
\usepackage{mathrsfs}
\usepackage[utf8]{inputenc}
\usepackage[english]{babel}
\usepackage{algorithm}
\usepackage{algpseudocode}
\usepackage{bookmark}
\DeclareGraphicsExtensions{.jpg,.pdf,.png}
\RequirePackage{amsopn}
\RequirePackage{amsfonts}
\RequirePackage{amsthm}
\theoremstyle{definition}

\newtheorem{lem}{Lemma}

\newtheorem{thm}{Theorem}

\usepackage{graphicx}
\begin{document}
\begin{frontmatter}
\title{A Novel Approach to Sparse Inverse Covariance Estimation Using Transform Domain Updates and Exponentially Adaptive Thresholding}
\author{Ashkan~Esmaeili and Farokh Marvasti*\footnote{Corresponding Author}\\Advanced Communications Research Institute (ACRI), and Electrical Engineering Department\\ Sharif University of Technology, Tehran, Iran}
\ead{esmaeili.ashkan@alumni.stanford.edu}
\begin{abstract}
~~~~Sparse Inverse Covariance Estimation (SICE) is useful in many practical data analyses. Recovering the connectivity, non-connectivity graph of covariates is classified amongst the most important data mining and learning problems. In this paper, we introduce a novel SICE approach using adaptive thresholding. Our method is based on updates in a transformed domain of the desired matrix and exponentially decaying adaptive thresholding in the main domain (Inverse Covariance matrix domain). In addition to the proposed algorithm, the convergence analysis is also provided. In the Numerical Experiments Section, we show that the proposed method outperforms state-of-the-art methods in terms of accuracy.
\end{abstract}
\end{frontmatter}
\section{Introduction}
Estimating certain matrices of interest is of significant importance in many applications such as Quantized Matrix Completion (QMC) \cite{esmaeili2019novel}, \cite{esmaeili2018recovering}. Matrix Completion (MC) in general deals with reconstruction of matrices with certain properties. In this paper, we focus on estimating Sparse Inverse Covariance Estimation (SICE). 
A rich and comprehensive literature exists on Sparse Covariance Estimation (SCE) and SICE. Precise SCE and SICE can prove to be profitable in practical scenarios as in \cite{moradipari2017using} where SCE helps the joint problem of supervised learning with MC. There are many different algorithms to estimate the covariance matrix or its inverse (the precision matrix). To avoid prolixity, we suffice to briefly go over a developed line of thought of methods and classify them based on their properties. In \cite{fan2016overview}, a handful of techniques in estimating the covariance matrix and the precision matrix are provided as an overview including Thresholding, Adaptive Thresholding, Generalized Thresholding, Positive Definiteness, Estimation of Covariance with Eigenvalue Constraints (EC2) method, factor-based model, Principal Orthogonal complEment Thresholding (POET) estimator, and Projected PCA for SCE.\\
Certain SICE methods are also discussed in \cite{fan2016overview} including $l1$-regularized Gaussian maximum likelihood estimator known as “graphical lasso”, column-by-column estimation method, the Tuning-Insensitive Graph Estimation and Regression (TIGER) method, the Estimating
Precision matrIx with Calibration (EPIC) method, Robust estimation techniques, and Symmetric SICE. One is invited to track the references listed in \cite{fan2016overview} to get familiar with the SICE timeline in detail. Among all methods prior to \cite{fan2016overview}, we discuss three literature algorithms for specific purposes.\\
In \cite {rolfs2012iterative}, the authors propose an iterative thresholding algorithm using a proximal gradient method (\textbf{G-ISTA}) for \textbf{SICE}, and claim that this method yields linear convergence rate, i.e., (\textbf{G-ISTA}) results in $O(log~ \epsilon)$ iteration complexity for the tolerance of $\epsilon$. The linear convergence rate of (\textbf{G-ISTA}) depends on the optimal point condition number (eigenvalue bound).
\\In \cite{NIPS2011_4266}, the authors propose a novel approach based on Newton's method employing a quadratic approximation leveraging the sparse structure. They provide strong guarantees in recovering the precision matrix. Their method is claimed to be superlinearly convergent.\\
In \cite{cai2011adaptive}, a thresholding procedure that is adaptive to the variability of individual entries is introduced. It is shown that the adaptive estimators are optimal while other thresholding methods achieve suboptimal convergence rates. It is worth noting that the algorithm we propose takes advantage of adaptive threshlding, but the thresholding approach is different from those in \cite{cai2011adaptive}, and \cite{rolfs2012iterative}.  We use exponentially decreasing thresholds due to the recovery accuracy reflected in phase transition curves in \cite{marvasti2012unified}, \cite{esmaeili2018iterative}, \cite{esmaeili2016comparison}. Another reason is that iterative method with adaptive thresholding (IMAT) has been shown to have noticeable accuracy even for recovering dense signals \cite{marvasti2012unified}, \cite{esmaeili2018iterative}. The extended version of the latter which includes MC based on adaptive singular value thresholding is also provided in \cite{esmaeili2016iterative} as one of our works.\\
In the recent two years, many authors have developed novel methods enhancing speed and performance in comparison to previous works (prior to 2016, when \cite{fan2016overview} is marked).
In \cite{zhang2018large} for example, the authors introdce a Newton-CG based mehtod which is too fast, even faster than QUIC considered as an accurate, fast, and prevalent method in the literature. Our claim in this paper is enhanced accuracy rather than reducded complexity. We will compare the accuracy of our method to state-of-the-art methods in the Numerical Experiments Section \ref{NE}. In our paper, we develop a novel method based on exponentially decaying adaptive threshold levels and performing updates (Gradient Ascent) in the transform domain.\\
SICE appears in many practical settings. Several applications of the SICE are leveraged in Neuroimaging applications such as analyzing brain connectivity pattern as in \cite{SCHOUTEN201646, JIE201684,KIM2015625,QIU201552, horn2014structural,VAROQUAUX2013405, ROSA2015493,Zhou_2014_CVPR, doi:10.1002/hbm.23870,Wee2014,RYALI20123852}.
Another application of SICE is introduced for Hyperspectral Image Classification in \cite{4156304}.
\\

\section{Problem Model}
The problem model taken into account in this section is as follows (Regularized log-det program):
\begin{equation}
\hat{\Theta}=\underset{\Theta}{\textrm{argmax}}{~\textrm{log}(\textrm{det}(\Theta))-\textrm{tr}(S\Theta)}-\lambda||\Theta||_1
\end{equation}\label{l1}
This problem model is considered by many authors. Assuming $\lambda$ is chosen so that the solution to the convex problem \ref{l1} is unique, the following algorithm is proposed to find the minimizer.
\subsection{The proposed Algorithm}

\begin{algorithm}[h!]\label{main}
	\small
	\caption{Sparse Inverse Covariance Estimation Using Transform Domain Updates and Exponentially Decaying Adaptive Thresholding ~\textbf{SICE-EDAT}}
\hspace*{\algorithmicindent} \textbf{Initialization:}  	 $\mathbf{\hat{S}}$ Low-rank empirical covariance matrix, $\rho_0$ Tikhonov Regularization constant, $\lambda$, \\$\quad \quad thr$ thresholding level, $\mu_0$ Initial Gradient Ascent step size, $d$ shrinkage factor, $L$ Maximum number of iterations\\
   \hspace*{\algorithmicindent} \textbf{Output:} $\mathbf{\hat{\Theta}}$

	\begin{algorithmic}[1]
	\State $k \leftarrow 0$
	 \Procedure {SICE-EDAT}{}
	\While {$k \leq L$}
\State $\mu_k\leftarrow \frac{thr}{\lambda}$
\State $\mathcal{\mathbf{M}}_k\leftarrow \mathbf{\Theta}_k^{-1}+\mu_k(\mathbf{\hat{S}}-\mathbf{\Theta}_k^{-1})$	
\State $\mathbf{\Theta}_k \leftarrow (\mathcal{\mathbf{M}}_k+\rho_k \mathcal{\mathbf{I}})^{-1}$
\State $\mathbf{\Theta}_k(\mathbf{\Theta}_k\leq thr) \leftarrow 0$
\State $thr \leftarrow thr \times d$
\EndWhile
\State $k \leftarrow k+1$
\EndProcedure\\
	\Return $\mathbf{\hat{\Theta}} \leftarrow \mathbf{\Theta}_k$
	\end{algorithmic} 
\end{algorithm} 
Our proposed algorithm is presented in Table \ref{main}. We first introduce the notations used in Table \ref{main} and afterwards different segments of the algorithm are explained.
First, the initialization is carried out. $\hat{\mathbf{S}}$ is the empirical covariance matrix which is derived from the i.i.d samples and is singular in practice due to fewer samples than the dimension (big data application). $\rho_0$ denotes the initial Tikhonov regularization constant. This parameter plays an important role and its importance will be addressed in the convergence analysis section. $thr$ is the thresholding level. $\mu_0$ denotes the initial step size of the update equation in the transform domain. $d$ is the decay factor for adaptive thresholding, and $k$ denotes the iteration index.
$L$ is the maximum number of iterations which determines the convergence criterion.  $\mu_k$ denotes the step size used for update of the matrix $\mathcal{\mathbf{M}}_k$ in the $k$-th iteration. The update is carried out in the transform domain, i.e., the correction update using the residual $\hat{\mathbf{S}}-\mathbf{\Theta}_k^{-1}$ is applied to the inverse of matrix $\mathbf{\Theta}_k$ in the $\mathbf{\Sigma}$ (covariance) domain. 
It is worth noting that this step of our algorithm is different from moving in the gradient ascent direction which is utilized in \cite{rolfs2012iterative} because it affects the update on the inverse of the desired $\mathbf{\Theta}$, i.e., the domain transform is carried out and the update is applied in the transformed domain. We believe this is more accurate than taking the gradient ascent direction as in \cite{rolfs2012iterative} because intuitively speaking $\mathbf{\Sigma}$ domain is the original source of the i.i.d samples and the covariance matrix is directly generated using the pure knowledge of the samples. Updating in the precision domain may introduce perturbation and deviation due to a probable ill-conditioned inversion; and therefore reduce the accuracy and robustness in the convergence. On the other hand, the original domain is more prone to the effect of perturbation. In fact, this attitude is quite similar to the methodology of IMAT, where the update is carried out in Fourier Domain for instance, and the thresholding occurs in the time domain to hold and preserve the originality of each operator for its own domain to prevent perturbations and transfrom-triggered deviations from the solution. Following the update, we revert the matrix to its own domain using the Tikhonov regularization, the role of which shall be explained in the convergence analysis Section \ref{CONV}.\\
In this step, the hard-thresholding is carried out in the $7$-th line of algorithm, i.e., entries with magnitude less than $thr$ are set to $0$.
Finally, the threshold level is decayed to pick up next important components in the following iteration in a similar fashion as done in \cite{marvasti2012unified}.
The choice of parameters will be elaborated upon in the Numerical Experiments section \ref {NE}.
The following theorem is about the convergence of our proposed method to the solution.
\begin{thm}
\textit{The Sequence of the estimators $\{\Theta_k\}_{k \in \mathcal{N}}$ generated by the algorithm \ref{main} converges to the unique solution of the convex problem in \ref{l1}}
\end{thm}
\begin{proof}
{Convergence Analysis}\label{CONV}
Consider the convex relaxation of the objective function in \eqref{l1} as modeled by many authors in the literature:
\begin{equation}
f(\mathbf{\Theta})={~\textrm{log}(\textrm{det}(\mathbf{\Theta})))-\textrm{tr}(S\mathbf{\Theta})}-\lambda||\mathbf{\Theta}||_1
\end{equation}
The update of $\mathbf{\Theta}^{k+1}$ in the $k+1-$ th iteration is achieved as follows:
\begin{equation}
\mathbf{\Theta}^{k+1}=\eta_{\mu_{k \lambda}}\{(\mathbf{\Theta}^{-1}+\mu_k(\hat{S}-\mathbf{\Theta}^{-1})+\rho_k\mathbf{I})\}^{-1}
\end{equation}\label{update}
For now, we ignore the term $\rho_k$ and assume $\rho_k=0$; we later explain the role of $\rho$ as the Tikhonov regularization parameter. In \ref{update}, $\eta$ is the hard thresholding operator with parameter $\mu_k \lambda$.
Since $\mathbf{\Theta}^*$ is the fixed point of the hard thresholding algorithm, i.e.:
\begin{equation}
\mathbf{\Theta}^*=\eta_{\mu_k\lambda}(\mathbf{\Theta}^*-\mu_k(S-(\mathbf{\Theta}^*)^{-1}))
\end{equation}
it immediately follows that:
\begin{align}
||\mathbf{\Theta}_{k+1}-\mathbf{\Theta}^*||_F^2=||\eta_{\mu_{k} \lambda}(\mathbf{\Theta}_k^{-1}+\mu_k(S-\mathbf{\Theta}_k^{-1}))^{-1}-\mathbf{\Theta}^*||_F^2
\\=||\eta_{\mu_{k} \lambda}(\mathbf{\Theta}_k^{-1}+\mu_k(S-\mathbf{\Theta}_k^{-1}))^{-1}-\eta_{\mu_k\lambda}(\mathbf{\Theta}^*-\mu_k(S-(\mathbf{\Theta}^*)^{-1}))||_F^2\\
\leq ||(\mathbf{\Theta}_k^{-1}+\mu_k(S-\mathbf{\Theta}_k^{-1}))^{-1}-\mathbf{\Theta}^*+\mu_k(S-(\mathbf{\Theta}^*)^{-1})||_F^2
\end{align}\label{contraction}
The inequality in \ref{contraction} is obtained using the non-expansiveness property of the hard thresholding operator. 
\begin{lem}
	Let $X(t)=A-tB$ then $\frac{d}{dt}X^{-1}|_{t=0}=A^{-1}BA^{-1}$
\end{lem}\label{makoos}
\begin{lem}
	Let $t$ denote a constant such that $t<<1$, and let $X(t)=A-tB$. The matrix $X(t)^{-1}$ can be approximated as 
	\begin{equation}
	X(t)^{-1}=A^{-1}+tA^{-1}BA^{-1}+\mathcal{O}(t^2)
	\end{equation}\label{taylor}
\end{lem}
	The proof is straightforward using the Laurent expansion of the matrix $X(t)$ and using Lemma \ref{makoos} for evaluating the first derivative.
	In \ref{taylor}, $\mathcal{O}(t^2)$ is a matrix with entries scale down with $t^2$.
\end{proof}

Now, assuming the step size of the algorithm $\mu_k$ is assigned such that the property $\mu_k<<1$ is held, we can approximate the first part of \ref{contraction} as:
\begin{equation}
(\mathbf{\Theta}_k^{-1}+\mu_k(\mathbf{S}-\mathbf{\Theta}_k^{-1}))^{-1}=\mathbf{\Theta}_k+\mu_k\mathbf{\Theta}_k(\mathbf{S}-\mathbf{\Theta}_k^{-1})\mathbf{\Theta}_k+\mathcal{O}(\mu_k^2)
\end{equation}
It follows that \ref{contraction} can be written as:

\begin{align}
||\mathbf{\Theta}_k-\mathbf{\Theta}^*+\mu_k(\mathbf{\Theta}_k \mathbf{S} \mathbf{\Theta}_k-\mathbf{\Theta}_k+\mathbf{S}-(\mathbf{\Theta}^*)^{-1})+\mathcal{O}(\mu_k^2)||_F^2\\=||\mathbf{\Theta}_k-\mathbf{\Theta}^*||_F^2+2\mu_k<\mathbf{\Theta}_k-\mathbf{\Theta}^*,\mathbf{\Theta}_k \mathbf{S} \mathbf{\Theta}_k-\mathbf{\Theta}_k+\mathbf{S}-(\mathbf{\Theta}^*)^{-1}>+\mathcal{O}(\mu_k^2),
\end{align}\label{badeq}where $<\mathbf{A},\mathbf{ B}>$ denotes the inner product of two matrices $\mathbf{A}, \mathbf{B}$ i.e. $Tr(\mathbf{A}\mathbf{B}^T)$.
Now, it suffices to show the inner product in \ref{badeq} takes a negative value (independent of $\mu_k$) to prove that \ref{badeq} is less than $||\mathbf{\Theta}_k-\mathbf{\Theta}^*||_F^2$ and as a result $||\mathbf{\Theta}_{k+1}-\mathbf{\Theta}^*||_F^2<||\mathbf{\Theta}_k-\mathbf{\Theta}^*||_F^2$. So, we focus on analyzing the inner product in \ref{badeq}. 
%
\section{Discussion on the Assumptions}
The following assumptions are considered to be held for our algorithm:
\begin{align}
<\mathbf{\Theta}_k,\mathbf{\Theta}^*>~ \leq ~\eta_k^2||\mathbf{\Theta}_k||_F^2
\\ 
<\mathbf{\Theta}_k\mathbf{S}\mathbf{\Theta}_k,\mathbf{\Theta}^*>\leq\eta_k^2 ||\mathbf{\Theta}_k||_F^2
\\
||\mathbf{S}-(\mathbf{\Theta}^*)^{-1}||_F^2 \leq \eta_k^2||\mathbf{\Theta}_k-\mathbf{\Theta}^*||_F^2
\end{align}
The first and second equations are conceptually related to the Cauchy-Schwartz inequality. In fact, $\eta_k^2$ is equivalent to the learning rate in our model.
The third inequality bounds the difference between the minimizer and the empirical covariance matrix. Big data models consider large dimensions in practice, and the fact that our problem targets big data scenarios helps us add this reasonable assumtion that for large dimensions, when several samples are accessible, the matrix $S$ is close to the oracle $\hat{S}$ (Central Limit Theorem) in Frobenius norm. 
We get back to the inner product in \ref{badeq}.

\begin{flalign}
<\mathbf{\Theta}_k-\mathbf{\Theta}^*,\mathbf{\Theta}_k \mathbf{S} \mathbf{\Theta}_k-\mathbf{\Theta}_k+\mathbf{S}-(\mathbf{\Theta}^*)^{-1}>\\=<\mathbf{\Theta}_k-\mathbf{\Theta}^*,\mathbf{\Theta}^*-\mathbf{\Theta}_k>+<\mathbf{\Theta}_k-\mathbf{\Theta}^*,\mathbf{\Theta}_k \mathbf{S} \mathbf{\Theta}_k-\mathbf{\Theta}^*+\mathbf{S}-(\mathbf{\Theta}^*)^{-1}>
\\=
-||\mathbf{\Theta}_k-\mathbf{\Theta}^*||_F^2+<\mathbf{\Theta}_k-\mathbf{\Theta}^*,\mathbf{S}-(\mathbf{\Theta}^*)^{-1}>+<\mathbf{\Theta}_k-\mathbf{\Theta}^*,\mathbf{\Theta}_k \mathbf{S} \mathbf{\Theta}_k-\mathbf{\Theta}^*> ~ \\ \leq -||\mathbf{\Theta}_k-\mathbf{\Theta}^*||^2+\eta_k^2 ||\mathbf{\Theta}_k-\mathbf{\Theta}^*||^2+<\mathbf{\Theta}_k-\mathbf{\Theta}^*,\mathbf{\Theta}_k\mathbf{S}\mathbf{\Theta}_k-\mathbf{\Theta}^*> ~~ 
\end{flalign}
For this to be negative it is required that:\\
\begin{flalign}
<\mathbf{\Theta}_k-\mathbf{\Theta}^*,\mathbf{\Theta}_k\mathbf{S}\mathbf{\Theta}_k-\mathbf{\Theta}^*>~ <||\mathbf{\Theta}_k-\mathbf{\Theta}^*||_F^2(1-\eta_k^2)
\end{flalign}\label{ineq}

Now, it is sufficient to show that $||\mathbf{\Theta}_k \mathbf{S} \mathbf{\Theta}_k-\mathbf{\Theta}^*||_F ~ < ~ (1-\eta_k^2)||\mathbf{\Theta}_k-\mathbf{\Theta}^*||_F$.
If this is shown, then using Cauchy-Schwartz inequality, the result in \ref{ineq} follows immediately.
\begin{lem}
	Tikhonov regularization in step $6$ of algorithm \ref{main} can be leveraged such that the inequality $||\mathbf{\Theta}_k \mathbf{S} \mathbf{\Theta}_k-\mathbf{\Theta}^*||_F^2 ~ < ~ (1-\eta_k^2)^2||\mathbf{\Theta}_k-\mathbf{\Theta}^*||_F^2$ holds.
\end{lem}\label{lmm}
\begin{proof}
The left-hand of inequality in lemma \ref{lmm} can be expanded as:
	\begin{equation}
	||\mathbf{\Theta}_k \mathbf{S} \mathbf{\Theta}_k-\mathbf{\Theta}^*||_F^2 = ||\mathbf{\Theta}_k\mathbf{S}\mathbf{\Theta}_k||_F^2+||\mathbf{\Theta}^*||_F^2-2<\mathbf{\Theta}_k\mathbf{S}\mathbf{\Theta}_k,\mathbf{\Theta}^*>
	\end{equation}
	The right-hand can be written as:
	\begin{equation}
 ||\mathbf{\Theta}_k||_F^2+||\mathbf{\Theta}^*||_F^2-2<\mathbf{\Theta}_k,\mathbf{\Theta}^*>+\mathcal{O}(\eta_k^2)
	\end{equation}
	It is enough to show that:
	\begin{equation}
	||\mathbf{\Theta}_k\mathbf{S}\mathbf{\Theta}_k||_F^2~<~||\mathbf{\Theta}_k||_F^2-2<\mathbf{\Theta}_k,\mathbf{\Theta}^*>+2<\mathbf{\Theta}_k\mathbf{S}\mathbf{\Theta}_k,\mathbf{\Theta}^*>+\mathcal{O}(\eta_k^2) = ||\mathbf{\Theta}_k||_F^2+\mathcal{O}(\eta_k^2)
	\end{equation}
	The last result is achieved using first and second assumptions. In fact, the interpretation is that $||\mathbf{\Theta}_k\mathbf{S}\mathbf{\Theta}_k||_F^2$ is smaller than $||\mathbf{\Theta}_k||_F^2$ s.t. the difference is $o(\eta_k^2)$ (smaller than $\mathcal{O}(\eta_k^2)$ in terms of order).
	At this stage of the convergence analysis, we open the discussion of Tihonov Regularization constant to see how it can be leveraged to guarantee the baseline inner product is negative. The update in the $6$-th line of the algorithm \ref{main} shrinks the eigenvalues of the matrix in the main domain with the coeffiecient $\frac{1}{1+\rho_k}$. 
In fact, intuitively, $\mathbf{\Theta}_k$s are shrunken versions of matrices $\tilde{\mathbf{\Theta}}_k$, and we want to show that:
	\begin{equation}
	\frac{1}{(1+\rho_k)^2}||\tilde{\mathbf{\Theta}_k}\mathbf{S}\tilde{\mathbf{\Theta}_k}||_F^2 < \frac{1}{(1+\rho_k)}||\tilde{\mathbf{\Theta}_k}||_F^2+\mathcal{O}(\eta_k^2)
	\end{equation}
	It is easily seen that $\rho_k$ can be tuned such that:
	\begin{equation}
	\frac{1}{1+\rho_k}||\tilde{\mathbf{\Theta}_k}\mathbf{S}\tilde{\mathbf{\Theta}_k}||_F^2 < (1-\mathcal{O}(\eta_k))||\tilde{\mathbf{\Theta}_k}||_F^2,
	\end{equation}
	Ignoring $\mathcal{O}(\eta_k^2)$, the result follows immediately.
\end{proof}

%
\section{Numerical Experiments}\label{NE}
In this section, we analyze the performance of our proposed method in two scenarios: $ \textrm{I}-$ Synthetic datasets, $\textrm{II}-$ Real datasets. First, we consider two synthetic scenarios introduced in \cite{NIPS2011_4266}. We briefly review how the posterior samples and the genuine precision matrices are generated as benchmarked in \citep{NIPS2011_4266}.
Two types of graph structures are considered with underlying Gaussian Markov Random Fields:



\begin{itemize}
\item
\textbf{Chain Graphs}: The ground truth precision matrix is assumed to sparsified as the following definition of $\Sigma^{-1}$ imposes:
 \[
    \left\{
                \begin{array}{ll}
                 \mathbf{\Sigma}^{-1}_{i,i-1}=-0.5 \\
                 \mathbf{\Sigma}^{-1}_{i,i}=1.25
                \end{array}
              \right.
  \]

\item 
\textbf{Graphs with Random Sparsity Structures}: Let $\mathbf{U}$ be a matrix with nonzero elements equal to $\pm1$, set $\mathbf{\Sigma}^{-1}$ to be $\mathbf{U}^T\mathbf{U}$ and then add a diagonal term to ensure it is positive definite. The number of nonzeros in $\mathbf{U}$ are controlled so that the resulting $\mathbf{\Sigma}^{-1}$
has approximately $10p$ nonzero elements. $n = \frac{p}{2}$ i.i.d. samples are generated from the corresponding GMRF distribution with  $\mathbf{\Sigma}^{-1}$.
\end{itemize}
\begin{table}
\small
\begin{flushleft}
\caption{Accuracy measurements of different methods (The left number in the columns related to each method denotes TPR and the right one shows FPR. (sparsity $5\%$))}
\begin{tabular}{ |c|c|c|c|c|c|c|c|c| }
\hline
Dataset & Size$(p)$  & Sparsity & Glasso & QUIC & CovApprox & Noncvx-CGGM & WISE & Proposed\\
\hline
\multirow{3}{3em}{Chain}&$1000$ & $\sim 1500$ & $1~~~0$&  $1~~~0$& $1~~~0$ & $0.994~~~0$ & $1~~~0$ & $1~~~0$\\ \cline{2-9}
&$4000$ &$\sim 6000$ & $1~~~0$ &$1~~~0$&$1~~~0$ & $1~~~0$ & $1~~~0$ & $1~~~0$ \\ \cline{2-9}
& $10000$ & $\sim 15000$ & $1~~~0$& $1~~~0$ & $1~~~0$ & $1~~~0$ & $1~~~0$ & $1~~~0$\\ 
\hline
\multirow{3}{4em}{Random} &$1000$ & $\sim 5000$& $0.61~~~0$&$0.78 ~~~ 0$ & $0.84~~0$ & $0.78~~0$ & $0.65~~~ 0$& $0.90~~0$\\  \cline{2-9}
& $4000$ & $\sim 20500$ & $0.78~~~0$ & $0.83 ~~ 0$ & $0.96~~0$ & $0.86~~~0$ & $0.73~~~0$ & $0.98~~0$\\  \cline{2-9}
&$10000$ & $\sim 45000$& $0.86~~~0$& $~0.90~~0$ & $0.99~~~0 $& $0.91~~0$ & $0.81~~~0$& $1~~0$ \\ 
\hline
\end{tabular}\label{T3}
\end{flushleft}
\end{table}

\begin{table}
\small
\begin{flushleft}
\caption{Accuracy measurements of different methods (The left number in the columns related to each method denotes TPR and the right one shows FPR. (sparsity $10\%$))}
\begin{tabular}{ |c|c|c|c|c|c|c|c|c| }
\hline
Dataset & Size$(p)$  & Sparsity & Glasso & QUIC & CovApprox & Noncvx-CGGM & WISE & Proposed\\
\hline
\multirow{3}{3em}{Chain}&$1000$ & $\sim 3000$ & $1~~~2\times10^{-4}$&  $1~~~3\times 10^{-5}$& $0.99~~~0$ & $0.90~~~0$ & $1~~~0$ & $1~~~0$\\ \cline{2-9}
&$4000$ &$\sim 12000$ & $1~~~0$ &$1~~~0$&$1~~~0$ & $0.99~~~0$ & $1~~~0$ & $1~~~0$ \\ \cline{2-9}
& $10000$ & $\sim 30000$ & $1~~~0$& $1~~~0$ & $1~~~0$ & $1~~~0$ & $1~~~0$ & $1~~~0$\\ 
\hline
\multirow{3}{4em}{Random} &$1000$ & $\sim 10000$& $0.56~~~2\times10^{-5}$&$0.69 ~~~ 4 \times 10^{-3}$ & $0.79~~0$ & $0.66~~7\times 10^{-3}$ & $0.44~~~ 0$& $0.84~~0$\\  \cline{2-9}
& $4000$ & $\sim 41000$ & $0.72~~~0$ & $0.83 ~~ 6\times 10^{-3}$ & $ 0.90~~0$ & $0.68~~~0$ & $0.51~~~0$ & $0.94~~0$\\  \cline{2-9}
&$10000$ & $\sim 90000$& $0.81~~~0$& $~0.90~~4\times10^{-6}$ & $0.96~~~0 $& $0.82~~0$ & $0.79~~~0$& $0.98~~0$ \\ 
\hline
\end{tabular}\label{T2}
\end{flushleft}
\end{table}

\begin{table}
\small
\begin{flushleft}
\caption{Accuracy measurements of different methods (The left number in the columns related to each method denotes TPR and the right one shows FPR. (sparsity $20\%$))}
\begin{tabular}{ |c|c|c|c|c|c|c|c|c| }
\hline
Dataset & Size$(p)$  & Sparsity & Glasso & QUIC & CovApprox & Noncvx-CGGM & WISE & Proposed\\
\hline
\multirow{3}{3em}{Chain}&$1000$ & $\sim 6000$ & $1~~~0$&  $1~~~0$& $1~~~0$ & $0.994~~~0$ & $1~~~0$ & $1~~~0$\\ \cline{2-9}
&$4000$ &$\sim 24000$ & $1~~~0$ &$1~~~0$&$1~~~0$ & $1~~~0$ & $1~~~0$ & $1~~~0$ \\ \cline{2-9}
& $10000$ & $\sim 60000$ & $1~~~0$& $1~~~0$ & $1~~~0$ & $1~~~0$ & $1~~~0$ & $1~~~0$\\ 
\hline
\multirow{3}{4em}{Random} &$1000$ & $\sim 20000$& $0.31~~~0$&$0.45 ~~~ 0$ & $0.62~~0$ & $0.41~~0$ & $0.32~~~ 0$& $0.65~~0$\\  \cline{2-9}
& $4000$ & $\sim 82000$ & $0.44~~~0$ & $0.60 ~~ 0$ & $ 0.76~~0$ & $0.48~~~0$ & $0.41~~~0$ & $0.78~~0$\\  \cline{2-9}
&$10000$ & $\sim 180000$& $0.52~~~0$& $~0.74~~0$ & $0.84~~~0 $& $0.56~~0$ & $0.44~~~0$& $0.85~~0$ \\ 
\hline
\end{tabular}\label{T4}
\end{flushleft}
\end{table}

\begin{table}
\small
\begin{center}
\caption{Time comparison of different methods vs. dimension variation (in secs)}
\begin{tabular}{ |c|c|c|c|c|c|c|c|c| }
\hline
Dataset & Size$(p)$  & Glasso & QUIC & CovApprox & Noncvx-CGGM & WISE & Proposed\\
\hline
\multirow{4}{3em}{Chain}&$100$ & $0.03$&  $0.002$& $0.003$ & $0.01$ & $0.001$ & $0.005$\\ \cline{2-8}
&$1000$ & $28$ &$1.9$&$2.9$ & $9$ & $2$ & $5$ \\ \cline{2-8}
&$4000$ & $252$ &$18$&$27$ & $80$ & $20$ & $43$ \\ \cline{2-8}
& $10000$ & $3.1E4$& $2.0E3$ & $3.1E3$ & $1.0E4$ & $2E3$ & $5.2E3$\\ 
\hline
\end{tabular}\label{T5}
\end{center}
\end{table}
In Table \ref{T2}, we compare our method to the recent state-of-the-art methods as well as glasso and the QUIC method. The comparison criteria are the True Positive Rate (TPR) and False Positive Rate (FPR) as defined in \cite{NIPS2011_4266}. 
The methods taken into account are \textbf{Glasso}, \textbf{QUIC}, \textbf{CovApprox}, \textbf{Noncvx-CGGM}, \textbf{WISE}, \textbf{SICE-EDAT} introduced in \cite{friedman2008sparse}, \cite{NIPS2011_4266}, \cite{siden2018efficient}, \cite{chen2018covariate}, \cite{nguyen2018distributionally}, and our current paper, respectively. 
We briefly summarize the data recorded in Table \ref{T2}. We can see that our method noticeably yields higher TPR in comparison to other methods as can be tracked in Table \ref{T2}.
It can be observed that the accuracy of our method outperforms other stat-of-the-art methods in Tables \ref{T2}\ref{T3}\ref{T4}. We emphasize again that the objective of this paper is to highlight the enhanced accuracy of our proposed method specifically for scenarios where the model is not too sparse. However, we have also provided Table \ref{T5} to show that our method computational complexity is comparable to other methods (although not the best). Our method is based on iterative matrix inversions. Depending on the case under study, the inversion could be done handled using matrix inversion lemma to reduce the complexity. 
\section{Conclusion}
In this paper, we have introduced a novel Sparse Inverse Covariance Estimation method based on exponentially decreasing threshold levels and updates in a transform domain to enhance accuracy in precision estimation. The exponentially decreasing threshold approach is derived from the IMAT method with strong recovery accuracy guarantees in sparse signal processing. Our proposed algorithm is also backed up with convergence analysis. The simulation results are provided to illustrate superior accuracy of our method compared to those of state-of-the-art methods in terms of True Positive Rate (TPR) and False Positive Rate (FPR) on different synthetic datasets.
\section*{References}
\bibliographystyle{elsarticle-num} 
\bibliography{reference.bib} 
\end{document}